\def\year{2022}\relax
\documentclass[letterpaper]{article} % DO NOT CHANGE THIS
\usepackage{aaai22}  % DO NOT CHANGE THIS
\usepackage{times}  % DO NOT CHANGE THIS
\usepackage{helvet}  % DO NOT CHANGE THIS
\usepackage{courier}  % DO NOT CHANGE THIS
\usepackage[hyphens]{url}  % DO NOT CHANGE THIS
\usepackage{graphicx} % DO NOT CHANGE THIS
\usepackage{natbib}  % DO NOT CHANGE THIS AND DO NOT ADD ANY OPTIONS TO IT
\usepackage{caption} % DO NOT CHANGE THIS AND DO NOT ADD ANY OPTIONS TO IT
\DeclareCaptionStyle{ruled}{labelfont=normalfont,labelsep=colon,strut=off} % DO NOT CHANGE THIS
\usepackage{algorithm}
\usepackage{algorithmic}
\usepackage{newfloat}
\usepackage{listings}
\floatstyle{ruled}
\newfloat{listing}{tb}{lst}{}
\floatname{listing}{Listing}
\title{Generalized Equivariance and Preferential Labeling\\
for GNN Node Classification}

\author{
        Zeyu Sun$^{\dag}$\ \ 
        Wenjie Zhang$^{\dag}$\ \ 
        Lili Mou$^{\ddag}$\ \ 
        Qihao Zhu$^{\dag}$\ \ 
        Yingfei Xiong$^{\dag}$\ \ 
        Lu Zhang$^{\dag}$\\
        }
\affiliations{
        $^{\dag}$Key Laboratory of High Confidence Software Technologies, MoE;\\
School of Computer Science, Peking University, 100871, P. R. China\\
        \{szy\_, zhang\_wen\_jie, zhuqh, xiongyf, zhanglucs\}@pku.edu.cn\\
          $^{\ddag}$Department of Computing Science, Alberta Machine Intelligent Institute (Amii)\\ University of Alberta,  
          Edmonton T6G 2E8, Canada\quad
        doublepower.mou@gmail.com
    }

\usepackage{xcolor}
\usepackage{xspace}
\usepackage{tikz}
\usepackage{amsfonts}
\usepackage{bm}
\usepackage{amsmath}
\usepackage{booktabs}
\usepackage{amsthm}
\usepackage{subfig}
\usepackage{soul}

\newtheorem{theorem}{Theorem}
\usetikzlibrary{patterns}

\newcommand{\techname}{{Preferential Labeling}\xspace}
\newcommand{\newcite}[1]{\citeauthor{#1}~(\citeyear{#1})}

\newcommand{\taua}{\tau_*^{(X)}}
\newcommand{\taub}{\tau_*^{(\pi(X))}}

\begin{document}

\maketitle

\begin{abstract}
Existing graph neural networks (GNNs) largely rely on node embeddings, which represent a node as a vector by its identity, type, or content. However, graphs with unattributed nodes widely exist in real-world applications (e.g., anonymized social networks). Previous GNNs either assign random labels to nodes (which introduces artefacts to the GNN) or assign one embedding to all nodes (which fails to explicitly distinguish one node from another). Further, when these GNNs are applied to unattributed node classification problems, they have an undesired equivariance property, which are fundamentally unable to address the data with multiple possible outputs.  
In this paper, we analyze the limitation of existing approaches to node classification problems. Inspired by our analysis, we propose a generalized equivariance property and a Preferential Labeling technique that satisfies the desired property asymptotically. Experimental results show that we achieve high performance in several unattributed node classification tasks.\footnote{The code and data are available at \\ https://github.com/zysszy/Preferential-Labeling}

\end{abstract}

\section{Introduction}
\label{ss:intro}

Graphs are a widely used type of data structure in computer science. A graph can be represented as $G = \langle V, E \rangle$, where $V$ is a set of \textit{nodes}, and $E$ is a set of node pairs known as \textit{edges} (directed or undirected). 
With the prosperity of deep learning techniques, graph neural networks (GNNs) are shown to be effective to various graph-related applications, such as program analysis~\cite{mou2016convolutional}, social networks~\cite{hamilton2017inductive}, knowledge graphs~\cite{hamaguchi2017knowledge}, molecule analysis~\cite{scarselli2008graph}, and the satisfiability (SAT) problem~\cite{zwj}. 

Existing GNNs highly rely on node embeddings, which are a vector representation of a node, typically based on its identity, type, or content. For example, a GNN for a knowledge graph typically embeds an entity/concept (e.g., a ``cat'' and a ``mammal'') as a vector~\cite{wang2018zero}, whereas a GNN for molecules embeds the atom (e.g., hydrogen and oxygen atoms) as a vector~\cite{scarselli2008graph}.  

In many applications, however, the nodes in a graph may not be attributed, and we call such a graph an \emph{unattributed graph}. A common scenario is that there is no attribute related to the nodes. For example, community detection for large-scale social networks may lack the identity information of nodes, i.e., a person, possibly due to privacy concerns~\cite{DBLP:conf/www/BackstromDK07}. Another scenario is that the attribute of a node is an artefact and captures no semantic meanings. Figure~\ref{fig:example} shows a graph that represents a propositional satisfiability (SAT) problem~\cite{selsam2018learning}, where $x_1$ and $c_1$ are arbitrary namings of the \textit{literals} (variable or its negations) and \textit{clauses} (disjunction of literals), and could be renamed without changing the nature of the formula. If such an identifier is represented by table look-up embeddings, it would become an \textit{artefact} in the GNN, because these embeddings do not represent common knowledge among different training samples. Nor do they generalize to new samples. 

To encode unattributed graphs, previous methods typically adopt an arbitrary labeling for nodes and represent them by embeddings~\cite{allamanis2017learning,WeiGDD20}. As mentioned, this introduces artefacts to GNNs. Recently, \newcite{selsam2018learning} have realized that such artefacts are undesired, and assign all nodes with the same embedding. 
However, this approach may suffer from the problem that the graph neural network becomes insensitive to the nodes.

In this work, we analyze unattributed node classification tasks, which require \textit{equivariance}, i.e., the change of node labels should be reflected correspondingly in the output. To address the mentioned problems, a na\"ive idea is to still assign different node embeddings, but to eliminate such artefacts by an ensemble of multiple labelings. For training, the labeling is randomly sampled every time we process a data sample; during inference, an ensemble of multiple random labelings is adopted for a sample. 
In this way, the nodes are distinguishable given any labeling, but such artefacts are smoothed out by the ensemble average. 

Our theoretical analysis, however, shows that such a na\"ive treatment does not work well for node classification.
An equivariant GNN is unable to solve equivariant node classification problems where multiple outputs are appropriate for an input graph. 

To this end, we propose a generalized equivariance property that is more suited to unattributed node classification. We further propose a \techname approach, which assigns multiple labelings during training but only updates the parameters with the best labeling. For inference, we also assign multiple labelings and make a prediction according to the best one. In this way, \techname asymptotically achieves our generalized equivariance property, and works well for multi-output equivariant node classification.

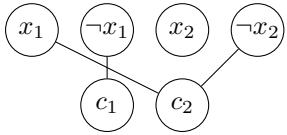
\begin{figure}
    \centering
    % \resizebox{}{}{}
    \begin{tikzpicture}
    \tikzstyle{place}=[circle,draw=black,minimum size=0.7cm,inner sep=0pt]
    \node [place] (x1) at (0, 0) {$x_1$};
    \node [place] (nx1) at (1, 0) {$\neg x_1$};
    \node [place] (x2) at (2, 0) {$x_2$};
    \node [place] (nx2) at (3, 0) { $\neg x_2$};
    \node [place] (c1) at (1, -1.0) {$c_1$}
        edge (nx1);
    \node [place] (c2) at (2, -1.0) {$c_2$}
        edge (x1)
        edge (nx2);
    % \draw [dotted] (x1) -- (nx1);
    % \draw [dotted] (x2) -- (nx2);
    \end{tikzpicture}
    \caption{A SAT formula can be represented by a graph, where a node $x_i$ is a \textit{literal} (a variable or its negation) and a node $c_i$ is a \textit{clause} (disjunction of literal nodes). The coresponding SAT formula is the conjunction of clauses, and in this example, it is $\neg x_1 \land (x_1 \lor \neg x_2)$.}
    \label{fig:example}
\end{figure}

We evaluated our approach on two unattributed node classification applications, maximum independent set solving (MIS) and propositional satisfiability solving (SAT).
Experimental results show that our approach successfully alleviates the limitations of existing GNNs when encoding unattributed graphs, where the number of errors drops by 39\% in the MIS problem and 76\% in the SAT problem. 

\section{Methodology}

In this section, we first present the problem formulation and analyze the equivariance property on unattributed graphs. 
Then, we present our \techname approach to address equivariant node classification.

\subsection{Problem Formulation}\label{sec:formulation}

A problem on unattributed graphs can be formalized as predicting output $Y$ given a graph $X$. A predicate function $H(X, Y)$, specific to a task, determines if $Y$ is appropriate for a given $X$. The predicate is true if and only if $Y$ is an appropriate solution for $X$. 

For an unattributed graph $G=\langle V, E\rangle$, the input can be fully represented by an adjacency matrix $X \in \left\{0, 1\right\}^{n \times n}$, where $n$ is the number of nodes. 
In a node classification task, the output is a matrix $ Y \in \mathbb R^{n \times k}$ for $n$ nodes and $k$ classes. 

% A graph neural network (GNN) takes the adjacent matrix $X$ and initial node embeddings $E_0$ as input and outputs $Y$.

To analyze how node indexes affect (or shall not affect) a GNN, we introduce the notation $S_n$ to represent the permutation group for $[n]$. 
Given $\pi \in S_n$, the action of $\pi$ on an unattributed graph $X\in\{0,1\}^{n\times n}$ is defined as $\left(\pi (X)\right)_{i,j} = X_{(\pi^{-1}(i)),(\pi^{-1}(j))}$, and its corresponding action on $Y \in \mathbb R^{n\times k}$ is given by $\left(\pi (Y)\right)_{i,c} = Y_{(\pi^{-1}(i)),c}$, i.e., $\pi$ denotes the same shuffle on the rows and columns of $X$, as well as the rows of $Y$. 
Here, $\pi$ is the mapping from node indexes to permuted indexes. Thus, $\pi^{-1}$ is retrieving the original node indexes in $X$ and $Y$ from the permuted indexes $i$ and $j$ in $\pi(X)$ and $\pi(Y)$. %Thus, $\pi^{-1}(\cdot)$ serves as the subscript of embeddings $\bm{e}$. 

\subsubsection{Equivariance.}
We now formulate the \textit{equivariance property} of node classification tasks. It essentially asserts that for any permutation $\pi \in S_n$, 
\begin{equation}\label{eqn:H}
H(X, Y) \text{\quad implies\quad} H(\pi (X), \pi (Y))
\end{equation}
That is to say, if we permute the order of nodes, the solution should be changed correspondingly.

Suppose for every $X$ there exists a unique $Y$ satisfying $H(X,Y)$, the mapping from $X$ to $Y$ can be modeled by a function $h$ and the equivariance property becomes the form that we commonly see
% \wenjiest{Similar to graph classification, the mapping from $X$ to $Y$ is usually a function, and we may denote $Y=h(X)$ if $H(X,Y)$ holds. Then, the equivariance property becomes}
\begin{equation}
h(\pi (X)) = \pi ( h(X) ) \label{equ:equivariance}
\end{equation}
for every permutation $\pi \in S_n$. 

In the above formulation, we define the equivariance property of a node classification task. In fact,
% invariance and 
equivariance can also be said in terms of GNN output $f(X)$,
given by 
\begin{equation}
f(\pi(X)) = \pi(f(X))\label{eqn:gnneq}
\end{equation}

\subsection{Limitations of Existing GNNs on Unattributed Graphs}\label{ss:limitEquiv}
We analyze the limitations of existing GNNs on unattributed graphs. As mentioned in Section~\ref{ss:intro}, previous approaches for unattributed graph either assign random labels to nodes~\cite{allamanis2017learning,WeiGDD20} or assign the same embedding to all nodes~\cite{li2018combinatorial,selsam2018learning}. 
When they are applied to unattributed node classifications, they suffer from at least one of the two limitations: 1) node distinction and 2) equivariance property. 

% However, they cannot distinguish one node from another. Further, when these approaches are applied to unattributed, they achieve undesired equivariance property 

\subsubsection{Node Distinction.}

We first consider distinguishing different nodes in a graph.
The state-of-the-art approaches~\cite{li2018combinatorial,selsam2018learning,zwj} assign all nodes with the same embedding, and thus, the model cannot distinguish different nodes effectively. Consider a common graph convolutional network (GCN), which learns the hidden representation for a node by encoding the node vector with the neighbors via a set of fully-connected layers. In the example given by Figure~\ref{fig:misexample}, all four nodes will have the same hidden representation, because every node is represented by the same embedding and all nodes have the same neighboring information.

\subsubsection{Equivariance Property.}
We now consider the equivariance property of node classification for unattributed graphs, which is believed to be important for various GNN applications~\cite{chen2018supervised,azizian2020expressive}. 

In node classification for unattributed graphs, if the node index changes, the output would change accordingly, shown in Eqn~(\ref{eqn:H}). Thus, it is tempting to design an equivariant GNN satisfying Eqn~(\ref{eqn:gnneq}) for node classification tasks, as suggested by~\newcite{DBLP:journals/tnn/WuPCLZY21}. 
Otherwise, the GNN would be sensitive to labeling artefacts~\cite{allamanis2017learning,WeiGDD20}, if it does not satisfy some form of equivariance. Previous work achieves the equivariance property~(\ref{eqn:gnneq}) by using the same embeddings for all nodes~\cite{li2018combinatorial,selsam2018learning,zwj}.

However, we hereby show that an equivariant GNN satisfying~(\ref{eqn:gnneq}) will fail on node classification problems, where multiple outputs are appropriate. In other words, the mapping from $X$ to $Y$ is not a function, and given an input $X$, there exists multiple $Y$ such that $H(X,Y)$ holds. Usually, GNN predicts one appropriate $Y$ by a function $Y=f(X)$.%Yet, an equivariant GNN approximates $Y$ by a function $Y=f(X)$.

\begin{figure}
    \centering
    \begin{tikzpicture}
    \tikzstyle{place}=[circle,draw=black,minimum size=0.5cm,inner sep=0pt]
    \node [place] (x1) at (0, 0) {$4$};
    \node [place] (x2) at (1, 0) {$3$}
        edge (x1);
    \node [place] (x3) at (1, 1) {$2$}
        edge (x2);
    \node [place] (x4) at (0, 1) {$1$}
        edge (x3)
        edge (x1);
    \end{tikzpicture}
    \caption{Graph $C_4$, a circle of length 4. This graph is auto-isomorhpic under $\pi:1\mapsto 2, 2\mapsto 3, 3\mapsto 4, 4\mapsto 1$.}
    \label{fig:misexample}
\end{figure}
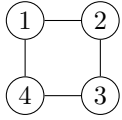
We show the drawback of  equivariant GNNs with an example of a non-trivial auto-isomorphic graph, i.e., there exists a non-identity permutation $\pi$ such that $\pi (X)= X$. If \eqref{eqn:gnneq} holds, then
$\pi (X) = X$ implies $f(X) = \pi (f(X))$. 
This means that GNN output must be the same for all corresponding nodes shuffled by $\pi$.

This, unfortunately, may be a bad solution for various tasks. Consider the maximum independent set (MIS) problem that selects the largest number of vertices that are not directly connected. In Figure~\ref{fig:misexample}, for example, $\{1,3\}$ is an MIS and $\{2,4\}$ is also an MIS. However, an equivariant GNN cannot predict any MIS in this example, because there exists a permutation $\pi$ (e.g., $\pi :1\mapsto 2, 2\mapsto3,3\mapsto4, 4\mapsto1$) essentially tying the output of all nodes.

It should be mentioned that we show the limitation by an example of non-trivial auto-isomorphism. Our analysis is suggestive to real applications, where graphs often have similar local structures.

\subsection{Our Solution}
\label{ss:solution}

We start with a na\"ive attempt to address both node distinction and equivariance in GNNs. With further analysis, we propose a generalized equivariance property and our \techname approach.
\subsubsection{A Na\"ive Attempt.}
To address the limitation of the node distinction, 
a na\"ive idea is still to assign node embeddings by randomly labeling the nodes, but to use an ensemble of different node labelings to eliminate artefacts~\cite{murphy2019relational,sato2021random}.
% For training, the node labels are randomly assigned in every epoch. 
% For inference, we 
For training, node labels are assigned randomly; this serves as a way of data augmentation, and can be thought of as training an ensemble over epochs. 
During inference, it assigns multiple random labels and uses an average ensemble for prediction. 

% However, this na\"ive idea still leaves the limitation in equivariant property.

% \subsubsection{Equivariant Property.}

However, it is not appropriate if we directly apply such a na\"ive idea to equivariant node classification. The standard cross-entropy training is essentially to 
\begin{equation*}
    \operatorname*{minimize}\limits_{\omega} \sum_{(X, Y) \in \mathcal D} \ \  \sum_{\pi \in S_n} \sum_{i=1}^n D_{\mathrm{KL}} \left( (\pi ( Y ))_i\ ||\  f_i (\pi (X); \omega) \right),
\end{equation*}
where $\omega$ denotes the trainable parameters in a GNN, $\mathcal D$ is the training set, and $D_{\mathrm{KL}}$ is the Kullback--Leibler divergence between the predictions and the ground truth. $(\pi(Y))_i$ and $f_i(\pi(X))$ are the $i$th row in a matrix, representing the target and predicted distributions of a node. 
However, the training objective will enforce the GNN to learn the same prediction of nodes under auto-isomorphism, because
 for every $\pi\in S_n$, the KL objective requires $\pi (Y) = f(\pi (X))$.  For some auto-isomorphic permutation $\tau$, i.e., $\tau (X) = X$, this implies $f(X) = f(\tau (X)) =  \tau (Y) $. Since $Y = f(X)$, we will have $f(X) = \tau( f(X))$. This is precisely the limitation that we have analyzed in Section~\ref{ss:limitEquiv}, namely, auto-isomorphism tying the prediction of a GNN.

To address this issue, we propose a desired generalized equivariance property. 
% Further, we also propose a \techname approach, that satisfies the desired property and addresses the limitation.

\subsubsection{Generalized Equivariance Property.} An equivariant GNN satisfying~(\ref{eqn:gnneq}) fails for equivariant node classification, because it unreasonably assumes the output is a function of input, i.e., approximating (\ref{eqn:H}) by~(\ref{equ:equivariance}). 

We relax this constraint for multi-output node classification and analyze the desired form of equivariance in this setting. We denote $\mathcal H(X)=\{Y: H(X, Y)\}$ be the set of all correct outputs given a graph $X$. 
The training set typically provides one groundtruth solution $Y_*\in \mathcal H(X_*)$ for a specific graph $X_*$, as usually one solution suffices in real applications (and this coincides with GNNs whose output is a function of input).

We would define $\mathcal H_*(\cdot)|_\mathcal{D}$ as a minimal equivariant subset of $\mathcal H(\cdot)$ such that $Y_*\in \mathcal H_*(X_*)$, with the domain restricted to $\mathcal D=\{X: X=\pi(X_*) \text{ for some } \pi\in S_n\}$.

This starts from defining \begin{align}\mathcal H_*(X_*) = \left\{ \gamma (Y_*) : \gamma \in S_n \text{\ and \ } \gamma (X_*) = X_* \right\}, \label{eqn:Fstar}
\end{align}
which essentially endorses multiple correct outputs other than the given $Y_*$ due to self-isomorphism $\gamma$.

Then, the equivariance property suggests $\mathcal H_*(X) = \pi \left( \mathcal H_*(X_*)\right)$, if $X = \pi (X_*)$ for some $\pi$. Here, we abuse the notation as $\pi(\mathcal H_*(X_*))\overset{\Delta}{=}\{\pi(Y):Y\in \mathcal H_*(X_*)\}$.

We would like to design a neural network predicting a correct solution, i.e., \begin{equation}
f(X_*) \in \mathcal H_*(X_*).
\label{equ:equivar}
\end{equation}
Due to the equivariance of $\mathcal H_*$, we have
\begin{equation}
f(\pi (X_*)) \in \mathcal H_*(\pi (X_*)) = \pi (\mathcal H_*(X_*)).
\label{equ:equpi}
\end{equation}

By the definition of $\mathcal H_*$ in \eqref{eqn:Fstar}, Eqn~\eqref{equ:equivar} implies that there exists $\gamma_1 \in S_n$ such that $\gamma_1 (X_*) = X_*$ and $f(X_*) = \gamma_1 (Y_*)$. Likewise, Eqn~\eqref{equ:equpi} implies that there exists $\gamma_2 \in S_n$ such that $\gamma_2 (X_*) = X_*$ and $f(\pi (X_*)) = \pi \gamma_2 (Y_*)$. Combining these and denoting $\gamma_2 \gamma_1^{-1}$ by $\gamma$, we see that there exists $\gamma\in S_n$ such that
\begin{equation}
\gamma (X_*)  = X_* \text{\quad and\quad} f(\pi (X_*)) = \pi \gamma (f(X_*)). \label{equ:gen-equ}
\end{equation}
We call \eqref{equ:gen-equ} the \textit{generalized equivariance property}. In fact, \eqref{eqn:gnneq} is a special case of \eqref{equ:gen-equ}, where $\gamma$ is an identity permutation. 
However, we relax~\eqref{eqn:gnneq} by allowing an additional auto-isomorphic permutation $\gamma$ in the solution space, and thus, it does not suffer from the limitation in Section~\ref{ss:limitEquiv}.

The analysis shows that an ideal GNN for unattributed node classification should satisfy \eqref{equ:gen-equ} rather than \eqref{equ:equivariance}.

\begin{figure}
    \centering
      \includegraphics[width=\linewidth]{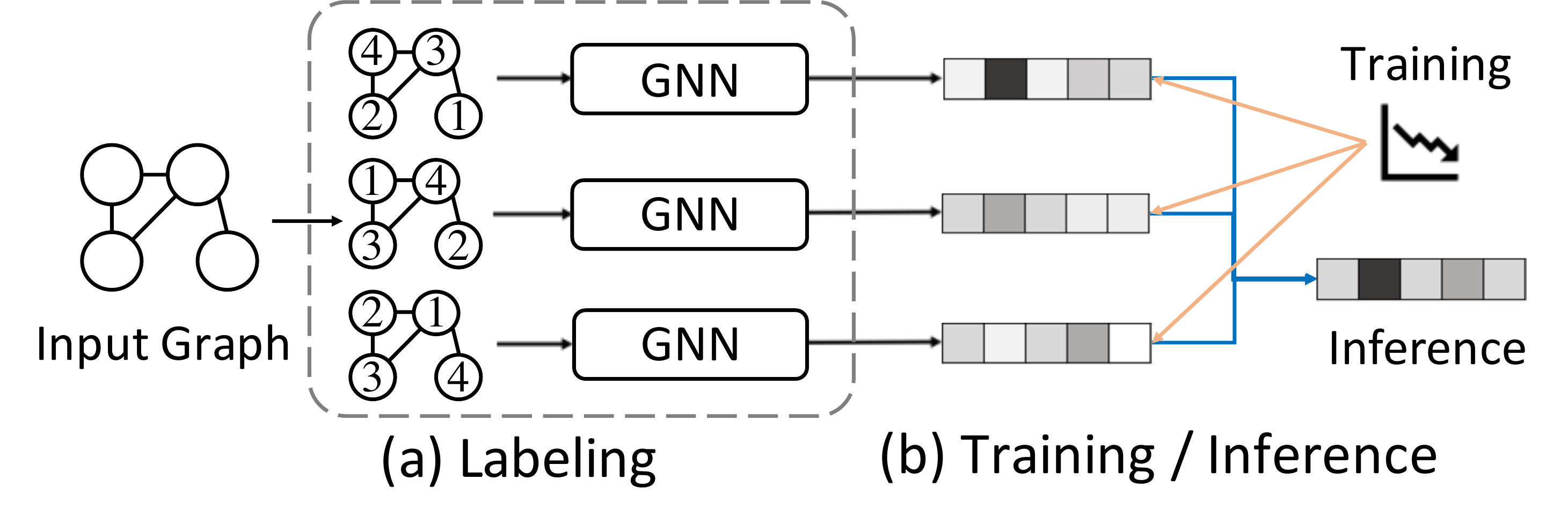}
    \caption{An overview of our \techname approach.}
    \label{fig:overview}
\end{figure}

\subsubsection{\techname.} Inspired by the above analysis, we propose a simple yet effective approach, \techname,  which asymptotically satisfies~\eqref{equ:gen-equ}. The overview of our approach is shown in Figure~\ref{fig:overview}. 
%\techname contains two parts: training and inference, which are introduced separately.

% \underline{\emph{\textbf{Training.}}}\label{sec:train} % + permutation v1~vn subset
For training, \techname assigns nodes with a random permutation of labels. A node is represented by a table-lookup embedding based on the assigned label. To satisfy~\eqref{equ:gen-equ}, we dynamically sample multiple label assignments for an input graph in each epoch, but only train the GNN with the best labeling (i.e., having the lowest loss). When processing the graph in the next epoch, we re-assign node labels and lookup for a (possibly) different preferred embedding by other random permutations. 

Formally, we allocate $\bm e_1, \cdots, \bm e_N$ as embedding parameters in our GNN model, where $N$ is the total number of embeddings; these embeddings have not been associated with any graph or node.
%In each epoch when processing a graph that has nodes $V=\{v_1,\cdots, v_n\}$ with $n\le N$, we randomly sample a permutation $\pi \in S_n$. Then, a node $v_i$ is represented by $\bm e_{\pi^{-1}(i)}$, where $\pi^{-1}$ retreives the embedding the graph neural network could be denoted by $\operatorname{GNN}(\bm e_{\pi^{-1}(1)}, \cdots \bm e_{\pi^{-1}(n)})$, where edge information is implicit in the GNN structure. 
In each epoch when processing a graph that has nodes $V=\{v_1,\cdots, v_n\}$ with $n\le N$, we randomly sample a permutation $\pi \in S_n$. By the convention of our paper, $\pi$ operates on adjacency matrices; consequently, a node $v_i$ is now represented by $\bm e_{\pi^{-1}(i)}$ for GNN processing.%. Then, a node $v_i$ is represented by $\bm e_{\pi(i)}$ and the graph neural network could be denoted by $\operatorname{GNN}(\bm e_{\pi(1)}, \cdots \bm e_{\pi(n)})$, where edge information is implicit in the GNN structure. 

We repeat this sampling process $K$ times, and compute the loss of these permutations when fed to GNN. Finally, we select the permutation that has the lowest loss as the final permutation for training.
In different training epochs, these permutations are re-sampled even for the same data sample. 
% Compared with an arbitrary fixed labeling of nodes~\cite{boldi2012injecting,allamanis2017learning}, our \techname reduces the artefacts of node labels. 

The training process can be described as
\begin{equation}
    \operatorname*{minimize}\limits_{\omega} \sum_{(X, Y) \in \mathcal D}  \min_{\pi \in S_n} \sum_{i=1}^n D_{\mathrm{KL}}  \left( (\pi (Y))_i \ || \  f_i (\pi (X); \omega) \right) ,\label{equ:pdltraining}
\end{equation}
where we train the model with the best labeling $\pi$, which is  a permutation of both $X$ and $Y$ matrices (Section~\ref{sec:formulation}).  %  We summarize the training process in Algorithm~\ref{alg:train}. 

Consider the inference of a $k$-way classification problem. The GNN with permutation $\pi_m$ outputs a probability as $\big (p_1^{(m)}, \cdots, p_k^{(m)}\big)$. %=\operatorname{GNN}\big(\bm e_{\pi_m^{-1}(1)}, \cdots, \bm e_{\pi_m^{-1}(k)}\big)$.
% , and this ensemble member predicts a label by $c^{(m)}=\operatorname{argmax}_i\{p_i^{(m)}\}$.
% Eventually, 
We pick the prediction that has the highest joint predicted probability (product of node probabilities). Eventually, our \techname approach predicts a label by $c=\operatorname{argmax}_i\{p_i\}$. % We summarize the inference process in Algorithm~\ref{alg:train}. %\todo{add a formula}

Our \techname does not suffer from the limitation of equivariance, because our network is not an equivariant function as in~\eqref{equ:equivariance}. However, we will achieve the generalized equivariance property \eqref{equ:gen-equ} asymptotically, because the labeling of $X$ is optimized out  (except for auto-isomorphism) during training and inference by taking the preferred $\pi$, detailed below.

\subsubsection{Theoretical Analysis.}

We show that the inference of our \techname asymptotically satisfies the generalized equivariance property for node classification, if we have enough sampled permutations. We also draw a connection with Expectation--Maximization (EM) algorithms.

During our inference, we assign multiple labelings to a graph, and pick the prediction that has the highest predicted probability as our output. Formally, we consider the joint predicted probability of a graph $X$ with labeling $\tau$ as
\begin{equation}
s(X, \tau) = \prod_{i=1}^n \ \max_{j=1, \cdots, k} f_{ij}(\tau(X)),
\end{equation}
where $f$ denotes a GNN function, outputing an $n\times k$ matrix ($n$: the number of nodes, $k$: the number of category). The $i$th row is the predicted distribution of the $i$th node. 

During inference, we have multiple labelings $\tau$. We pick the best one that maximizes $s(X,\tau)$, given by
\begin{equation}
\tau_*^{(X)} = \operatorname*{argmax}_{\tau \in S_n} s(X, \tau). %s(\tau (X)).
\label{eqn:tau}
\end{equation}

The prediction of our \techname is
\begin{equation}
\hat Y(X) = \left(\tau_*^{(X)}\right)^{-1} \left( f\left(\tau_*^{(X)} (X)\right)\right). \label{eqn:y}
\end{equation}
The formula follows our convention of adjacency matrix representations. When we perform node labeling, we permute the indexes of both $X$ and $Y$ by $\tau_*^{(X)}$. After GNN processing, we need an inverse permutation $\big(\tau_*^{(X)}\big)^{-1}$ to obtain the predictions for $X$, because our prediction should be corresponding to the original graph $X$, rather than $\tau_*^{(X)}(X)$. 

\medskip
\begin{theorem}
$\hat Y(\cdot)$ achieves generalized equivariance, i.e., for any graph $X$ and permutation $\pi \in S_n$, there exists $\gamma\in S_n$ such that $\gamma (X)  = X \text{\;and\;} \hat Y(\pi (X)) = \pi \gamma (\hat Y(X))$.
\end{theorem}

\begin{proof}
Consider any graph $X$ and permutation $\pi$.
Replacing $X$ by $\pi (X)$ in Eqn~\eqref{eqn:tau}, we have
\begin{equation}
% \tau_*^{(\pi (X))} = \operatorname*{argmax}_{\tau \in S_n} s(\pi(X), \tau \pi (X)). \label{eqn:pitau}
\tau_*^{(\pi (X))} = \operatorname*{argmax}_{\tau \in S_n} s(\pi(X), \tau). \label{eqn:pitau}
\end{equation}
Notice that  Eqns~\eqref{eqn:pitau} and Eqn~\eqref{eqn:tau} are essentailly the same problem, and that their optima should be achieved by the same element, i.e., $\tau_*^{(X)}(X)=\tau_*^{(\pi(X))}\pi(X)$.

This essentially means that the two permutations $\taua$ and $\taub\pi$ yield the same result on $X$, implying that they are the same, except for an auto-isomorphic permutation. In other words, there exists $\gamma$ such that $\gamma (X) = X$ and
$\taua = \taub\pi\gamma$, which can be rearranged as
\begin{equation}
\tau_*^{(\pi ( X))} = \tau_*^{(X)} \gamma^{-1} \pi^{-1}. \label{eqn:taurelation}
\end{equation}

Replacing $X$ by $\pi (X)$ in Eqn~\eqref{eqn:y}, we have
\begin{align}
\hat Y(\pi (X)) & = (\tau_*^{(\pi (X))})^{-1} (f(\tau_*^{(\pi (X))} \pi (X)))\label{eqn:res1} \\
& = \pi \gamma (\tau_*^{(X)})^{-1} (f(\tau_*^{(X)} \gamma^{-1} \pi ^{-1} \pi (X)))\label{eqn:res2} \\
& = \pi \gamma (\tau_*^{(X)}) ^{-1} ( f(\tau_*^{(X)} (X)))\label{eqn:res3} \\
& = \pi \gamma ( \hat Y(X)), \label{eqn:res4}
\end{align}
where \eqref{eqn:res2} is due to the substitution with~\eqref{eqn:taurelation}; \eqref{eqn:res3} is due to the cancellation of $\pi^{-1}\pi$ and the auto-isomorphism of $\gamma$, i.e., $\gamma(X)=X$; and \eqref{eqn:res4} is due to the definition of $\hat Y$ in~\eqref{eqn:y}.
\end{proof}

Our Preferential Labeling is also related to EM algorithms.
\begin{theorem}
The training of Preferential Labeling in (\ref{equ:pdltraining}) is a hard Expecatation--Maximization algorithm with a uniform prior on $S_n$.
\end{theorem}
\begin{proof} The labeling $\pi$ can be thought of as a latent variable in the task of mapping a graph $X$ to output $Y$.
The $\min$ operator in (\ref{equ:pdltraining}) is to seek $\pi\in S_n$ maximizing
the likelihood of output given input and the latent labeling, denoted by $P(Y|X,\pi)$. With the assumption of uniform prior $P(\pi)$ for $\pi\in S_n$, this is equivalent to cross-entropy training with one latent labeling $\pi$ that maximizes the posterior $P(\pi |X,Y)\propto P(Y|X,\pi)P(\pi)$, known as hard EM~\cite{hardEM}.
\end{proof}

This easy theorem provides further intuition on our \techname approach. EM algorithms are known for handling multi-modal mixtures of distributions, similar to  multi-output node classification. The training of our \techname is also analogous to the E-step, which determines the fitness of the sample to a mixture component. Our approach adopts a hard EM variant that selects a single best permutation, because full marginalization of $S_n$ is intractable.
Also, we assume a uniform prior for $S_n$, which is particularly suitable for eliminating labeling artefacts.

\section{Experiments}
We conducted experiments on two node classification tasks. We chose state-of-the-art or standard GNN architectures, but compare our approach with various embedding strategies.

\subsection{Competing Methods}
\label{sec:baselines}

\subsubsection{Static Labeling.} The static labeling assigns an embedding based on the identity of a node (e.g., $x_1$ and $c_1$ in Figure~\ref{fig:example}), although such identity does not represent meaningful semantics in different samples. Static learning is widely applied in previous work~\cite{boldi2012injecting,allamanis2017learning}.

\subsubsection{Same Embedding.} This baseline assigns all nodes in the unattributed graph with the same embedding. This is adopted in previous work~\cite{li2018combinatorial,zwj} to eliminate the artefacts of node labeling.

\subsubsection{Random Labeling.} The random labeling assigns an embedding randomly during training and inference. This is a special case of our approach with $K = 1$, and no actual preferential training is performed.

\subsubsection{Degree Feature.} An intuitive way to encode a node without labeling artefact is by its degree information, which captures some local information of the node. In this baseline, we use $1 / (d_{v}+1)$ as a one-dimensional, non-learnable embedding feature, where $d_v$ is the degree of node $v$. We use the above formula so that the feature is in $(0, 1]$.

\subsubsection{Degree Ranking Embedding.} The drawbacks of using degree information as a single numerical feature are its low-dimensionality and non-learnability. In this baseline, we extend the idea by embedding the ranking of node degrees. Specifically, we sort all nodes by the degrees in descending order, and a node having $i$th largest degree is encoded by $i$th embedding vector $\bm e_i$.

\begin{table}[!t]
    \centering
    \resizebox{.93\linewidth}{!}{
        \begin{tabular}{rlr}
            \toprule
            \textbf{Row} \# &\textbf{GCN~\cite{li2018combinatorial}}&\textbf{Accuracy}\\
            \midrule
            1&Same & 75.59\% \\
            2&Degree Feature & 73.22\%\\
            3&Degree Ranking Embedding & 71.58\% \\
            4&Static Labeling & 74.57\% \\
            % Static \& Ensemble-10 (Avg) & 85.6\% $\pm$ 3.0\% \\
            5&Random Labeling & 75.28\%  \\
            \midrule
            6&Preferential Labeling-10 & \textbf{85.04}\% \\
            % Dynamic \& Ensemble-10 (Avg) & \textbf{92.0}\% $\pm$ \textbf{2.8}\%\\
            \bottomrule
        \end{tabular}
    }
    \caption{The results for MIS solving. ``\techname-10'' indicates 10 random labelings in both training and inference.}
    \label{table:resultmis}
\end{table}

\begin{table*}[!t]
    \centering
    \resizebox{0.78\textwidth}{!}{
        \begin{tabular}{rlrrrrr}
            \toprule
            &&\multicolumn{5}{c}{\textbf{Error Rate}}\\
            \cmidrule{3-7}
            \textbf{Row} \#& \textbf{NLocalSAT~\cite{zwj}} & \textbf{Test-5} & \textbf{Test-10} & \textbf{Test-20} & \textbf{Test-40} & \textbf{Avg.}\\
            \midrule
            1&Same & 5.26\% & 8.17\% & 15.03\% & 27.62\% & 14.02\%\\
            2&Degree Feature  & 5.31\% & 8.37\% & 14.25\% & 24.94\% & 13.22\%\\
            3&Degree Ranking Embedding  & 5.45\% & 10.23\% & 16.17\% & 28.04\% & 14.97\%\\
            4&Static & 6.11\% & 9.86\% & 16.89\% & 28.88\% & 15.44\%\\
            % Static \& Avg-Voting (20)& 6.12\% & 5.90\% & 6.29\% & 6.97\% \\             
            % Static \& Avg-Voting (50) & 6.06\% & 5.83\% & 6.29\% & 6.97\% \\
            % Static \& Ensemble (20) & 6.76\% & 5.90\% & 6.29\% & 6.97\% \\            
            5&Static \& Inference-10 (Averaging) & 5.00\% & 8.77\% & 15.74\% & 29.70\% & 14.80\%\\
            6&Static \& Inference-10 (Max Prob.) & 1.77\% & 3.65\% & 7.86\% & 16.22\% & 7.38\%\\
            7&Random & 3.38\% & 6.17\% & 12.70\% & 23.66\% & 11.48\%\\
            % Dynamic \& Ensemble (20) & 6.47\% & 4.89\% & 4.94\% & 5.36\%\\
            8&Random \& Inference-10 (Averaging) & 3.39\% & 6.07\% & 12.42\% & 23.34\% & 11.31\%\\
            9&Random \& Inference-10 (Max Prob.) & 2.72\% & 5.03\% & 11.37\% & 22.06\% & 10.30\%\\
            \midrule
            10&Preferential Labeling-10 (Max Prob.) & \textbf{1.13}\% & \textbf{1.68}\% & \textbf{1.81}\% & \textbf{5.24}\% & \textbf{2.47}\%\\
            
            \bottomrule
        \end{tabular}
    }
    \caption{The results for SAT solving. ``Test-$k$'' indicates a test set where each sample has $k$ variables. ``Inference-$m$'' indicates $m$ random labelings during inference. 
    %``Preferential Labeling-10'' indicates $10$ random labelings as ensemble members in both training and inference of \techname.
    }
    \label{table:basicResults}
\end{table*}

\subsection{Experiment \uppercase\expandafter{\romannumeral1}: MIS Solving}
We first evaluate our \techname on solving the maximum independent set (MIS). 
In graph theory, an \textit{independent set} is a set of nodes without any edge. An independent set is \textit{maximum}, if it has the largest number of nodes among all independent sets. 
%A graph may contain multiple maximal independent sets; the largest maximal independent set is \textit{maximum}.
MIS solving is an NP-hard problem that aims to find out a maximum independent set from a graph. 

For an input graph, the goal of the GNN in this task is to predict a binary label for each node, deciding whether a node is in the MIS.
To induce an MIS from model predictions, we use a simple search algorithm. We first sort all nodes in descending order based on the predicted probability that the node is in the MIS. Then, we iterate over nodes in order and select the top node into the MIS;  its neighbors are removed from the list. The process is iterated until we have processed the entire node list. In this way, the selected nodes are guaranteed to be an independent set. Our evaluation determines whether it is maximum.

\textbf{Model.} 
In this experiment, we adopt the state-of-the-art model~\cite[GCN;][]{li2018combinatorial} for MIS solving.  \newcite{li2018combinatorial} use the same embedding for all nodes. The model contains 20 graph convolutional layers, which are regularized by dropout with rate of 0.1.  For the hidden size of all layers used in this model, we set it to 128. For training, we use Adam~\cite{kingma2014adam} to train the model with learning rate $10^{-4}$ on a single Titan RTX.

\textbf{Dataset.} We follow the data synthesis process in previous work~\cite{li2018combinatorial} and generate 173,751, 20,000, 20,000 graphs for training, development, and test, respectively. The number of nodes in a graph is generated from a uniform distribution $U[100, 150]$. 

\textbf{Results.} Table~\ref{table:resultmis} shows the results for MIS solving. Since our post-processing ensures the output is an independent set, the performance evaluation focuses on whether it is maximum. If the predicted set has the same number of nodes as the groundtruth MIS, we say the model solves this MIS correctly; otherwise, the model makes an error. 
% which represents how many percentages of MIS solving problems the model can solve. 

We observe that Static Labeling (Row~4) has low performance as it introduces labeling artefacts.  Same and Random Labelings (Rows~1 and~5) eliminate such artefacts and perform better. 
 The Degree Feature (Row~2) and Degree Ranking Embedding (Row 3) suffer from the limitation of the equivariance property mentioned in Section~\ref{ss:limitEquiv}, and
 perform worse than other baselines. 
 
 By contrast, \techname (Row 6) is able to eliminate labeling artefacts, and at the same time, achieve the desired generalized equivariance property in Eqn~\eqref{equ:gen-equ}. Its performance is higher than all competing approaches, with the number of errors dropping by 39\% from the best baseline. 
% We further analyze the performance of the proposed Generalized Equivariance Property on this model. 
% The approach that achieves this property (Row 6) has better performance than the approaches that fail to achieve Generalized Equivariance Property (Rows 1-5). 

\subsection{Experiment \uppercase\expandafter{\romannumeral2}: SAT Solving}

We further evaluate \techname on the SAT solving problem. % as another unattributed node classification task.
The propositional satisfiability problem (SAT) is one of the most fundamental problems in computer science. 
A propositional formula is said to be \textit{satisfiable}, if there exists an assignment of propositional variables to either True or False that makes the formula True; such assignment is known as a \textit{certificate}. 

We consider a specific setting of SAT solving, where the given formula is known to be satisfiable, and the goal is to predict a certificate, i.e., whether a variable should be assigned with True or False. This is a key step in SAT solvers.

\textbf{Model.}
The GNN model and settings are generally adopted from the state-of-the-art NLocalSAT~\cite{zwj}.

A SAT formula is represented as a bipartite graph, where a node is either a clause or a literal (see Figure~\ref{fig:example} for an example). The nodes are represented by identifiers, which are labeling artefacts.
In our experiment, we applied a convolution-based NLocalSAT model~\cite{zwj}, which achieves state-of-the-art performance for SAT solving. \newcite{zwj} use the same embedding for all clause/literal nodes. The model has $16$ convolutional layers, regularized by a dropout rate of $0.1$. In our model, we perform \techname for literals and clauses from respective candidate labelings/embeddings. 

\textbf{Dataset.} We used the SAT dataset in~\newcite{zwj}.
The training and development sets contain 500K and 396K SAT formulas, respectively. The number of variables in a formula is generated from a uniform distribution $U[10, 40]$, whereas the number of clauses is generated from $U[2, 6]$. For testing, the dataset contains four sets of different levels of difficulty. Specifically, the number of variables in a formula is 5, 10, 20, or 40 in each test set, denoted by Test-5, Test-10, Test-20, or Test-40, each containing 40K, 20K, 10K, or 5K test formulas.
%The number of clauses is all generated from $U[2, 6]$.

\textbf{Results.}
Table~\ref{table:basicResults} shows the results for SAT solving, where the performance of a model is evaluated by the formula-level error rate, i.e., if the predicted assignment does not make the formula true, we say the model makes an error. 

As mentioned, Static Labeling (Row~4) introduces artefacts of node identities, whereas using the same embedding (Row~1) is unable to distinguish different nodes well. The Degree Feature baseline alleviates these issues and performs better than Rows~1 and~4 in this task. However, they do not perform well in general.

%We investigate the effectiveness of the inference process of \techname, which can actually be applied to different base models (Rows~6 and~9). We observe that the performance is consistently better than the original method (Rows~4 and~7). 

We analyze the performance of an equivariant GNN that satisfies Eqn~\eqref{equ:equivariance}. This can be achieved by Random Labeling for training (Rows~8 and~9), by explicitly introducing averaging ensembles during inference (Rows~5 and~8) as the na\"ive attempt introduced in Section~\ref{ss:solution}, or by using the same embedding (Row~1) or the degree embedding (Rows~2 and~3).
Their performance, although better than Static Labeling (Row 4), appears inadequate. 

We then evaluate the effect of Preferential Labeling in the inference stage, applied to different baseline models. This
relaxes \eqref{eqn:gnneq} but satisfies generalized equivariance \eqref{equ:gen-equ} during inference. We see the error rates (Rows~6 and~9) are considerably lower than the GNN as an equivariant function. 

Moreover, \techname explicitly reduces labeling artefacts during training. With the inference algorithm controlled, our approach largely outperforms  training with Static and Random Labelings (Row~10 vs.~Rows~6 and~9).

\begin{figure}[!t]
    \centering
      \includegraphics[width=0.8\linewidth]{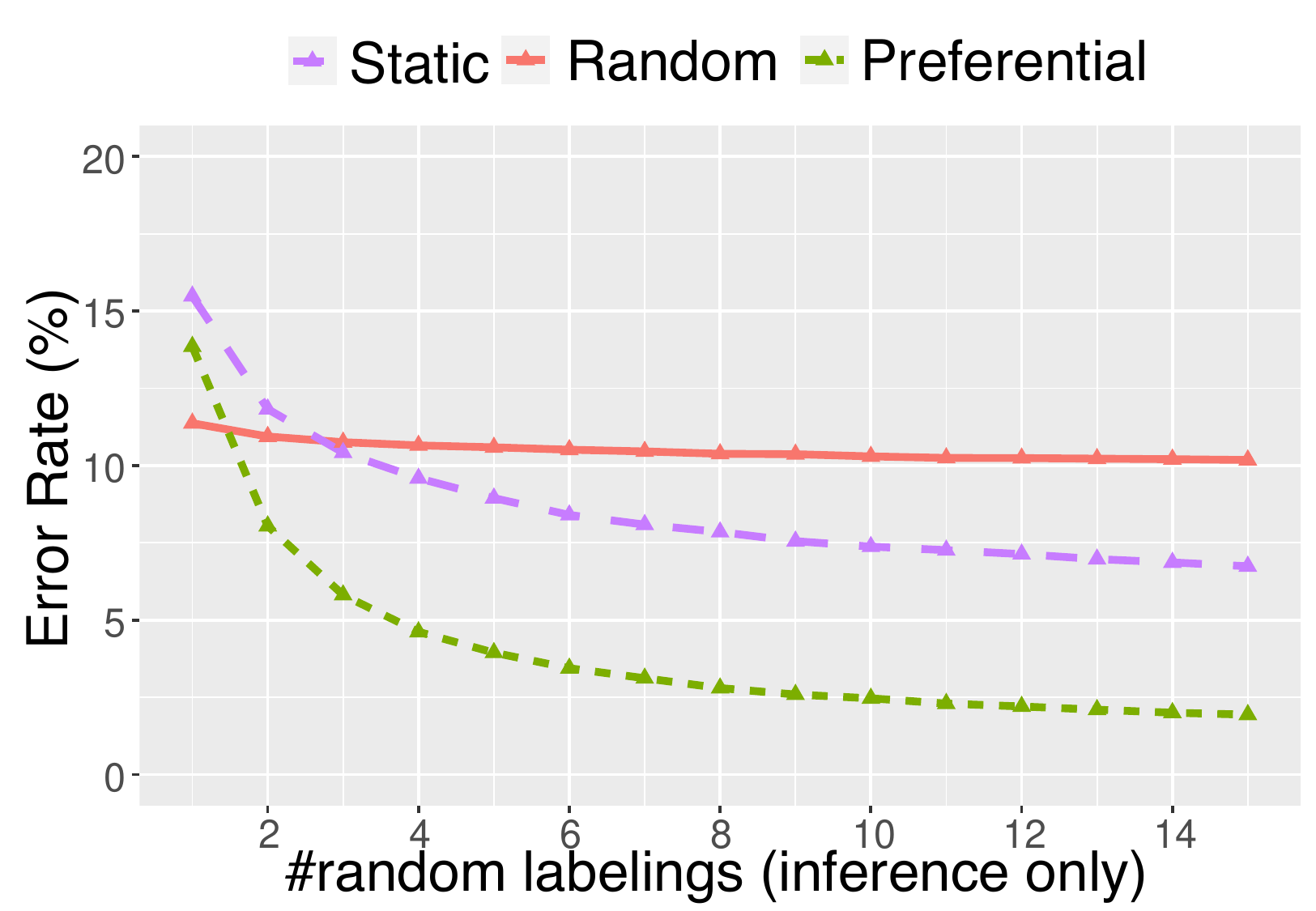}
    \caption{Error rate versus the number of random labelings during inference. We compare the embedding strategies for training, and all variants use the labeling with the maximum predicted probability for inference.}
    \label{fig:inference}
\end{figure}
\begin{figure}[!t]
    \centering
      \includegraphics[width=0.8\linewidth]{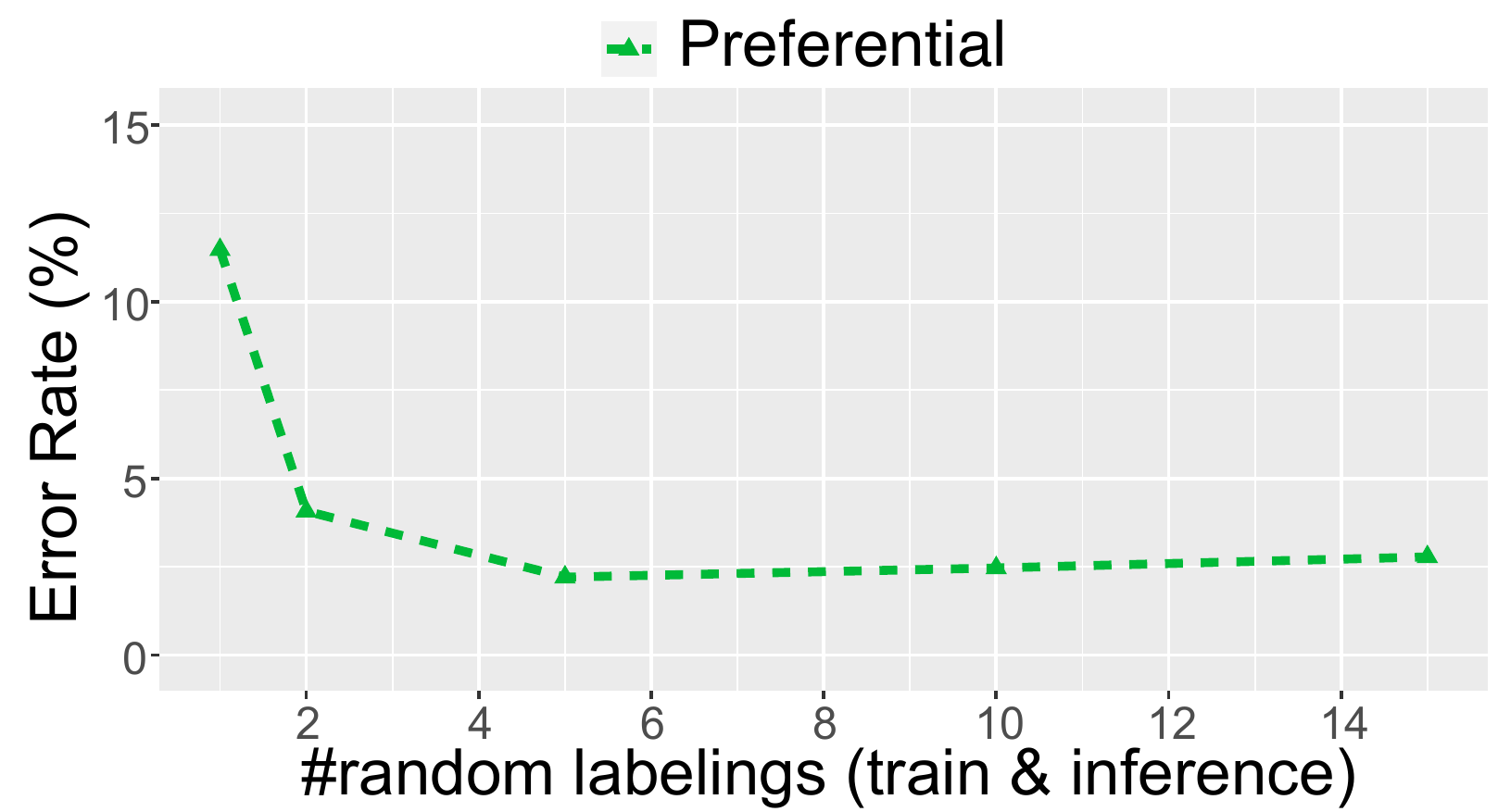}
    \caption{Error rate versus the number of random labelings during both training and inference.}
    \label{fig:training}
\end{figure}

We analyze how the number of random labelings affects model performance during inference, shown in Figure~\ref{fig:inference}. We observe that all models achieve higher performance with more labelings. However, the improvement for Random Labeling is marginal, as it suffers from the limitation of an equivariant function in a fundamental way, regardless of the number of labelings. 

Static Labeling and \techname do not achieve good performance if the number of labelings is small (e.g., $\le 2$). A plausible explanation is that the few labelings during inference may not comply with the training, resulting in high variance. However, the performance improves largely when we have more labelings, as these models are able to relax the function equivariance but achieve generalized equivariance asymptotically. Specifically, our proposed \techname is consistently better than Static Labeling by a large margin, as our model is explicitly trained with the best permutation in a hard EM fashion.

Finally, we analyze in Figure~\ref{fig:training} how the number of both training and inference labelings affect the performance of \techname. In this figure, we use the average error rate (the results on other settings are available at the link in Footnote~1). Results show that if the number of labeling is one, it reduces to random labeling, yielding poor performance. However, the error rate drops significantly when the number increases, and becomes stable when the number is great than or equal to 5. This shows that our approach could still be applied when computational resources are restricted.

\section{Related Work}

% \paragraph{Graph neural networks (GNNs).} 
Graph neural networks (GNNs) have been widely researched in recent years~ \cite{scarselli2008graph,battaglia2018relational}. GNNs have a variety of applications in different domains, ranging from social networks~\cite{kipf2016semi,hamilton2017inductive}, knowledge graphs~\cite{hamaguchi2017knowledge}, and programming source code~\cite{mou2016convolutional,mou2018tree}.

A common GNN architecture is the graph convolutional network (GCN,~\citeauthor{kipf2016semi}, \citeyear{kipf2016semi}).
Recently, researchers have designed various GNN architectures, such as gated graph neural network~\cite{li2015gated}, graph attention networks~\cite{DBLP:conf/iclr/VelickovicCCRLB18}, and Transformer-based GNN~\cite{DBLP:conf/aaai/CaiL20}. However, the focus of this paper is not the architecture design. Rather, we focus on node representations in unattributed graphs.   

To represent a node in graphs, DeepWalk~\cite{DBLP:conf/kdd/PerozziAS14} learns the node presentation by predicting the neighbors in an unsupervised way. Such pretraining-style node embedding does not generalize to new graphs. 
In certain applications, researchers use domain-specific information as labels. For example,  a node in a knowledge graph is represented by text~\cite{DBLP:conf/ijcai/LinLS16}  % [TBCNN, L2S, TreeGen, Allamanis, LLS].
% In certain applications, researchers apply a node by contextual information. For example, a code identifier in an abstract syntax tree is represented by subtokens in \newcite{DBLP:conf/sigsoft/AllamanisBBS15} 
and a variable in code analysis/generation tasks is often denoted by its name or subtokens~\cite{allamanis2017learning,suntreegen,xiong2018learning}. %in \newcite{DBLP:conf/ijcai/LinLS16}.

The embedding of nodes in unattributed graphs is not extensively addressed in previous literature. Existing work generally applies either an arbitrary labeling~\cite{allamanis2017learning} or the same embedding~\cite{li2018combinatorial,selsam2018learning,zwj}, which suffer from several limitations as discussed in this paper.
To address this, we propose %\techname for invariant graph classification tasks, and Preferential 
\techname for unattributed node classification tasks. 

% \paragraph{Graph Isomorphism Test.} 
Our \techname is also related to, but different from \newcite{xu2018powerful} and \newcite{garg2020generalization}, where they show that GNNs are limited in determining graph isomorphism. We instead showed equivariant GNNs are limited in solving one-to-many equivariant problems, and further proposed the desired generalized equivariant property. Moreover, our \techname does not suffer from the above limitation, because it assumes nodes are unlabeled, but we have (preferential) labelings.

\section{Conclusion}
In this paper, we address the task of node classification of unattributed graphs. We analyze the limitations of existing GNNs, showing that an equivariant GNN may not solve an equivariant node classification task, when multiple outputs are correct. 
We propose a generalized equivariance property, which allows an additional auto-isomorphic permutation.
Based on our analysis, we further propose \techname that samples multiple permutations and uses the best one for training and inference; theoretical analysis shows that our \techname achieves the desired generalized equivariance property asymptotically. 
We conducted extensive experiments on MIS solving and SAT solving tasks to demonstrate the effectiveness and generality of our approach.

\section*{Acknowledgments}
The work is supported in part by the National Key Research and Development Program of China under Grant No. 2019YFE0198100, the Innovation and Technology Commission of HKSAR under Grant No. MHP/055/19, National Natural Science Foundation of China under Grant No. 61922003, and the Natural Sciences and Engineering Research Council of Canada (NSERC) under grant No.~RGPIN2020-04465. Lili Mou is supported in part by the Amii Fellow Program, the Canada CIFAR AI Chair Program, and a donation from DeepMind. This research is also supported by Compute Canada (www.computecanada.ca).

\bibliography{example_paper}

\begin{thebibliography}{31}
\providecommand{\natexlab}[1]{#1}

\bibitem[{Allamanis, Brockschmidt, and Khademi(2018)}]{allamanis2017learning}
Allamanis, M.; Brockschmidt, M.; and Khademi, M. 2018.
\newblock Learning to represent programs with graphs.
\newblock In \emph{International Conference on Learning Representations}.

\bibitem[{Azizian et~al.(2020)}]{azizian2020expressive}
Azizian, W.; et~al. 2020.
\newblock Expressive power of invariant and equivariant graph neural networks.
\newblock In \emph{International Conference on Learning Representations}.

\bibitem[{Backstrom, Dwork, and Kleinberg(2007)}]{DBLP:conf/www/BackstromDK07}
Backstrom, L.; Dwork, C.; and Kleinberg, J.~M. 2007.
\newblock Wherefore art thou r3579x?: Anonymized social networks, hidden
  patterns, and structural steganography.
\newblock In \emph{Proceedings of the 16th International Conference on World
  Wide Web}, 181--190.

\bibitem[{Battaglia et~al.(2018)Battaglia, Hamrick, Bapst, Sanchez-Gonzalez,
  Zambaldi, Malinowski, Tacchetti, Raposo, Santoro, Faulkner
  et~al.}]{battaglia2018relational}
Battaglia, P.~W.; Hamrick, J.~B.; Bapst, V.; Sanchez-Gonzalez, A.; Zambaldi,
  V.; Malinowski, M.; Tacchetti, A.; Raposo, D.; Santoro, A.; Faulkner, R.;
  et~al. 2018.
\newblock Relational inductive biases, deep learning, and graph networks.
\newblock \emph{arXiv preprint}.

\bibitem[{Boldi et~al.(2012)Boldi, Bonchi, Gionis, and
  Tassa}]{boldi2012injecting}
Boldi, P.; Bonchi, F.; Gionis, A.; and Tassa, T. 2012.
\newblock Injecting uncertainty in graphs for identity obfuscation.
\newblock \emph{Proceedings of the VLDB Endowment}, 1376--1387.

\bibitem[{Cai and Lam(2020)}]{DBLP:conf/aaai/CaiL20}
Cai, D.; and Lam, W. 2020.
\newblock Graph transformer for graph-to-sequence learning.
\newblock In \emph{Proceedings of the AAAI Conference on Artificial
  Intelligence}, 7464--7471.

\bibitem[{Chen, Li, and Bruna(2018)}]{chen2018supervised}
Chen, Z.; Li, L.; and Bruna, J. 2018.
\newblock Supervised community detection with line graph neural networks.
\newblock In \emph{International Conference on Learning Representations}.

\bibitem[{Garg, Jegelka, and Jaakkola(2020)}]{garg2020generalization}
Garg, V.; Jegelka, S.; and Jaakkola, T. 2020.
\newblock Generalization and representational limits of graph neural networks.
\newblock In \emph{International Conference on Machine Learning}, 3419--3430.

\bibitem[{Hamaguchi et~al.(2017)Hamaguchi, Oiwa, Shimbo, and
  Matsumoto}]{hamaguchi2017knowledge}
Hamaguchi, T.; Oiwa, H.; Shimbo, M.; and Matsumoto, Y. 2017.
\newblock Knowledge transfer for out-of-knowledge-base entities: {A} graph
  neural network approach.
\newblock In \emph{Proceedings of the 26th International Joint Conference on
  Artificial Intelligence}, 1802--1808.

\bibitem[{Hamilton, Ying, and Leskovec(2017)}]{hamilton2017inductive}
Hamilton, W.~L.; Ying, Z.; and Leskovec, J. 2017.
\newblock Inductive representation learning on large graphs.
\newblock In \emph{Proceedings of the 31st International Conference on Neural
  Information Processing Systems}, 1024--1034.

\bibitem[{Kingma and Ba(2015)}]{kingma2014adam}
Kingma, D.~P.; and Ba, J. 2015.
\newblock Adam: {A} method for stochastic optimization.
\newblock In \emph{{International Conference on Learning Representations}}.

\bibitem[{Kipf and Welling(2017)}]{kipf2016semi}
Kipf, T.~N.; and Welling, M. 2017.
\newblock Semi-supervised classification with graph convolutional networks.
\newblock In \emph{International Conference on Learning Representations}.

\bibitem[{Li et~al.(2016)Li, Tarlow, Brockschmidt, and Zemel}]{li2015gated}
Li, Y.; Tarlow, D.; Brockschmidt, M.; and Zemel, R.~S. 2016.
\newblock Gated graph sequence neural networks.
\newblock In \emph{{International Conference on Learning Representations}}.

\bibitem[{Li, Chen, and Koltun(2018)}]{li2018combinatorial}
Li, Z.; Chen, Q.; and Koltun, V. 2018.
\newblock Combinatorial optimization with graph convolutional networks and
  guided tree search.
\newblock In \emph{Proceedings of the 32nd International Conference on Neural
  Information Processing Systems}, 537--546.

\bibitem[{Lin, Liu, and Sun(2016)}]{DBLP:conf/ijcai/LinLS16}
Lin, Y.; Liu, Z.; and Sun, M. 2016.
\newblock Knowledge representation learning with entities, attributes and
  relations.
\newblock In \emph{Proceedings of the Twenty-Fifth International Joint
  Conference on Artificial Intelligence}, 2866--2872.

\bibitem[{Mou and Jin(2018)}]{mou2018tree}
Mou, L.; and Jin, Z. 2018.
\newblock \emph{Tree-Based Convolutional Neural Networks: Principles and
  Applications}.
\newblock Springer.

\bibitem[{Mou et~al.(2016)Mou, Li, Zhang, Wang, and Jin}]{mou2016convolutional}
Mou, L.; Li, G.; Zhang, L.; Wang, T.; and Jin, Z. 2016.
\newblock Convolutional neural networks over tree structures for programming
  language processing.
\newblock In \emph{Proceedings of the Thirtieth AAAI Conference on Artificial
  Intelligence}, 1287--1293.

\bibitem[{Murphy et~al.(2019)Murphy, Srinivasan, Rao, and
  Ribeiro}]{murphy2019relational}
Murphy, R.; Srinivasan, B.; Rao, V.; and Ribeiro, B. 2019.
\newblock Relational pooling for graph representations.
\newblock In \emph{International Conference on Machine Learning}, 4663--4673.

\bibitem[{Perozzi, Al{-}Rfou, and Skiena(2014)}]{DBLP:conf/kdd/PerozziAS14}
Perozzi, B.; Al{-}Rfou, R.; and Skiena, S. 2014.
\newblock {DeepWalk}: Online learning of social representations.
\newblock In \emph{Proceedings of the 20th ACM SIGKDD International Conference
  on Knowledge discovery and data mining}, 701--710.

\bibitem[{Samdani, Chang, and Roth(2012)}]{hardEM}
Samdani, R.; Chang, M.-W.; and Roth, D. 2012.
\newblock Unified expectation maximization.
\newblock In \emph{Proceedings of the 2012 Conference of the North American
  Chapter of the Association for Computational Linguistics: Human Language
  Technologies}, 688--698.

\bibitem[{Sato, Yamada, and Kashima(2021)}]{sato2021random}
Sato, R.; Yamada, M.; and Kashima, H. 2021.
\newblock Random features strengthen graph neural networks.
\newblock In \emph{Proceedings of the 2021 SIAM International Conference on
  Data Mining}, 333--341.

\bibitem[{Scarselli et~al.(2009)Scarselli, Gori, Tsoi, Hagenbuchner, and
  Monfardini}]{scarselli2008graph}
Scarselli, F.; Gori, M.; Tsoi, A.~C.; Hagenbuchner, M.; and Monfardini, G.
  2009.
\newblock The graph neural network model.
\newblock \emph{{IEEE} Transactions on Neural Networks}, 20(1): 61--80.

\bibitem[{Selsam et~al.(2019)Selsam, Lamm, B{\"{u}}nz, Liang, de~Moura, and
  Dill}]{selsam2018learning}
Selsam, D.; Lamm, M.; B{\"{u}}nz, B.; Liang, P.; de~Moura, L.; and Dill, D.~L.
  2019.
\newblock Learning a {SAT} solver from single-Bit supervision.
\newblock In \emph{{International Conference on Learning Representations}}.

\bibitem[{Sun et~al.(2020)Sun, Zhu, Xiong, Sun, Mou, and Zhang}]{suntreegen}
Sun, Z.; Zhu, Q.; Xiong, Y.; Sun, Y.; Mou, L.; and Zhang, L. 2020.
\newblock TreeGen: {A} tree-based {transformer} architecture for code
  generation.
\newblock In \emph{Proceedings of the AAAI Conference on Artificial
  Intelligence}, 8984--8991.

\bibitem[{Velickovic et~al.(2018)Velickovic, Cucurull, Casanova, Romero,
  Li{\`{o}}, and Bengio}]{DBLP:conf/iclr/VelickovicCCRLB18}
Velickovic, P.; Cucurull, G.; Casanova, A.; Romero, A.; Li{\`{o}}, P.; and
  Bengio, Y. 2018.
\newblock Graph attention networks.
\newblock In \emph{International Conference on Learning Representations}.

\bibitem[{Wang, Ye, and Gupta(2018)}]{wang2018zero}
Wang, X.; Ye, Y.; and Gupta, A. 2018.
\newblock Zero-shot recognition via semantic embeddings and knowledge graphs.
\newblock In \emph{Proceedings of the IEEE Conference on Computer Vision and
  Pattern Recognition}, 6857--6866.

\bibitem[{Wei et~al.(2020)Wei, Goyal, Durrett, and Dillig}]{WeiGDD20}
Wei, J.; Goyal, M.; Durrett, G.; and Dillig, I. 2020.
\newblock {LambdaNet}: {P}robabilistic type inference using graph neural
  networks.
\newblock In \emph{International Conference on Learning Representations}.

\bibitem[{Wu et~al.(2021)Wu, Pan, Chen, Long, Zhang, and
  Yu}]{DBLP:journals/tnn/WuPCLZY21}
Wu, Z.; Pan, S.; Chen, F.; Long, G.; Zhang, C.; and Yu, P.~S. 2021.
\newblock A comprehensive survey on graph neural networks.
\newblock \emph{IEEE Transactions on Neural Networks and Learning Systems},
  32(1): 4--24.

\bibitem[{Xiong and Wang(2021)}]{xiong2018learning}
Xiong, Y.; and Wang, B. 2021.
\newblock {L2S}: A framework for synthesizing the most probable program under a
  specification.
\newblock \emph{TOSEM: ACM Transactions on Software Engineering and
  Methodology}.

\bibitem[{Xu et~al.(2019)Xu, Hu, Leskovec, and Jegelka}]{xu2018powerful}
Xu, K.; Hu, W.; Leskovec, J.; and Jegelka, S. 2019.
\newblock How powerful are graph neural networks?
\newblock In \emph{{International Conference on Learning Representations}}.

\bibitem[{Zhang et~al.(2020)Zhang, Sun, Zhu, Li, Cai, Xiong, and Zhang}]{zwj}
Zhang, W.; Sun, Z.; Zhu, Q.; Li, G.; Cai, S.; Xiong, Y.; and Zhang, L. 2020.
\newblock {NLocalSAT}: {B}oosting local search with solution prediction.
\newblock In \emph{Proceedings of the Twenty-Ninth International Joint
  Conference on Artificial Intelligence}, 1177--1183.

\end{thebibliography}

\end{document}